\documentclass{article}

\usepackage{microtype}
\usepackage{graphicx}
\usepackage{subcaption}
\usepackage{booktabs} 
\usepackage{multirow}
\usepackage{makecell}
\usepackage{enumitem}
\usepackage{placeins}
\usepackage{hyperref}

\usepackage[preprint]{icml2026}

\usepackage{amsmath}
\usepackage{amssymb}
\usepackage{mathtools}
\usepackage{amsthm}

\usepackage[capitalize,noabbrev]{cleveref}

\theoremstyle{plain}
\newtheorem{theorem}{Theorem}[section]
\newtheorem{proposition}[theorem]{Proposition}

\theoremstyle{definition}
\newtheorem{definition}[theorem]{Definition}

\theoremstyle{remark}

\usepackage[textsize=tiny]{todonotes}

\icmltitlerunning{CAPS: Unifying Attention, Recurrence, and Alignment in Transformer-based Time Series Forecasting}

\begin{document}

\twocolumn[
  \icmltitle{CAPS: Unifying Attention, Recurrence, and Alignment in Transformer-based Time Series Forecasting}



  \icmlsetsymbol{equal}{*}

  \begin{icmlauthorlist}
    \icmlauthor{Viresh Pati}{1}
    \icmlauthor{Yubin Kim}{1}
    \icmlauthor{Vinh Pham}{1}
    \icmlauthor{Jevon Twitty}{1}
    \icmlauthor{Shihao Yang}{1}
    \icmlauthor{Jiecheng Lu}{1}

  \end{icmlauthorlist}

  \icmlaffiliation{1}{Georgia Institute of Technology, Atlanta, GA, USA}

  \icmlcorrespondingauthor{Jiecheng Lu}{jliu414@gatech.edu}
  \icmlcorrespondingauthor{Shihao Yang}{shihao.yang@isye.gatech.edu}

  \icmlkeywords{Machine Learning, ICML}

  \vskip 0.3in
]



\printAffiliationsAndNotice{}

\begin{abstract}

This paper presents $\textbf{CAPS}$ (Clock-weighted Aggregation with Prefix-products and Softmax), a structured attention mechanism for time series forecasting that decouples three distinct temporal structures: global trends, local shocks, and seasonal patterns. Standard softmax attention entangles these through global normalization, while recent recurrent models sacrifice long-term, order-independent selection for order-dependent causal structure. CAPS combines SO(2) rotations for phase alignment with three additive gating paths --- Riemann softmax, prefix-product gates, and a Clock baseline --- within a single attention layer. We introduce the Clock mechanism, a learned temporal weighting that modulates these paths through a shared notion of temporal importance. Experiments on long- and short-term forecasting benchmarks surpass vanilla softmax and linear attention mechanisms and demonstrate competitive performance against seven strong baselines with linear complexity. Our code implementation is available at this \href{https://github.com/vireshpati/CAPS-Attention}{link}.

\end{abstract}

\section{Introduction}
Time series forecasting has important applications across domains such as finance \cite{GIANTSIDI2025104719}, weather and climate modeling \cite{kim2024comprehensive}, and traffic flow prediction \cite{LIU2026132269}. Effective forecasting requires reconciling three inductive biases. First, estimating stable global structure (e.g., levels) benefits from order-independent aggregation that weights observations by relevance without respect to temporal order \cite{cleveland1990stl, wu2021autoformer}. Second, modeling local shocks and transient behavior requires order-dependent propagation, where contributions decay causally over time \cite{BoxJenkins2015, zhou2021informer}. Third, capturing seasonality and repeating patterns demands alignment-sensitive comparison based on relative phase instead of absolute position \cite{wu2021autoformer, zhou2022fedformer}. Classical statistical models achieve this separation through explicit decomposition into level, trend, seasonal, and residual components \cite{HyndmanAthanasopoulos2021, cleveland1990stl}, but require parameter tuning and cannot leverage large datasets. 

Modern deep learning models promise to learn these decompositions from data. Recent work has adopted Transformer-based architectures for time series forecasting, motivated by their success in sequence modeling \cite{Wen2023Survey}. Transformer-based models process temporal relationships primarily through the attention mechanism \cite{vaswani2017attention}. Attention mechanisms excel at content-based selection, by which we mean aggregation where weights depend on query-key similarity rather than solely on temporal distance \cite{vaswani2017attention, gu2024mamba}. However, a single softmax normalization entangles the classic decomposition of trends, shocks, and seasonality, a challenge emphasized in recent surveys of statistical and deep learning-based time series forecasting \cite{Lim2021TSsurvey, Wen2023Survey, zeng2023transformers}. While attention can model global structure, local shocks, and seasonal patterns in practice, this coupling prevents the mechanism from simultaneously (i) maintaining stable trend estimates across growing histories, (ii) applying time-step-independent decay to isolated shocks, and (iii) aligning observations based on relative phase offset. We discuss these challenges in Section \ref{sec:motivating-example}.

To compensate, existing architectures introduce auxiliary components such as series decomposition \cite{wu2021autoformer, zhou2022fedformer}, frequency-domain mixing \cite{wu2023timesnet}, and local patching \cite{nie2023a}. However, these mechanisms operate outside the core attention kernel and require manual specification and tuning of decomposition windows or decay rates. Empirical studies further show that many such variants fail to consistently outperform simple linear projections \cite{zeng2023transformers}. The limitations may not arise from representational capacity, but could arise from how past information is weighted, normalized, and aggregated.

A complementary line of work addresses this issue by replacing softmax normalization with order-dependent aggregation. Linear attention mechanisms \cite{katharopoulos2020transformers} and gated recurrent models \cite{orvieto2023resurrecting,yang2023gated} propagate information through associative updates, enabling linear-time computation and improved stability over long horizons. Structured state-space models explicitly parameterize temporal evolution through learned transition operators, achieving strong performance on long-sequence modeling tasks \cite{gu2022efficiently,gu2024mamba}. These approaches successfully decouple order-dependent propagation from global normalization, but they achieve this by discarding content-based selection since weights depend on temporal distance rather than query-key similarity. 

In this work, we show how a single attention layer can simultaneously support all three inductive biases: order-independent global aggregation, order-dependent causal propagation, and alignment-sensitive comparison. We introduce \textbf{CAPS}, a three-part attention mechanism that decouples these components through three additive gating paths grounded in group-theoretic decomposition. CAPS encodes alignment through block-diagonal SO(2) rotations that implement RoPE \cite{su2024roformer}, models causal decay through diagonal SPD prefix-product gates, and captures global structure through a weighted Riemann softmax function called the Clock mechanism. Unlike existing approaches that either couple all three mechanisms through softmax normalization, separate them by discarding content-based selection, or approximate them by stacking attention layers, CAPS captures all three capabilities within a unified layer.

 Our contributions can be summarized as follows:
\begin{enumerate}[label=(\roman*)]
  \item We prove that a single layer of softmax attention couples all tokens through global normalization (Proposition~\ref{prop:coupling}), preventing token-independent decay for transient dynamics. This is a fundamental limitation for time series with local shocks.
    
    \item We introduce CAPS, a structured attention mechanism that decouples global aggregation, causal decay, and phase alignment through three additive gating paths---Riemann softmax, prefix-product gates, and a learned Clock---unified within a single linear attention layer.
    
    \item CAPS outperforms Linear Attention with RoPE on all ten datasets and achieves competitive performance against strong baselines, with average rank 2.3 and four first-place finishes in the summarized results. Ablations show the gains come specifically from the linear setting: the three-path mechanism improves linear attention by 6.1\%, but has essentially no effect with softmax (0.7\%).
\end{enumerate}

\section{Background}
\label{sec:background}

\subsection{Problem Setup and Notation}

We consider multivariate time series forecasting. Given an input sequence $(x_1, \ldots, x_L)$ with $x_t \in \mathbb{R}^C$ across $C$ channels, the goal is to predict future values $(\hat{y}_{L+1}, \ldots, \hat{y}_{L+H})$ over a horizon $H$ for $C_{\mathrm{target}} \leq C$ target channels. We use $h_t \in \mathbb{R}^d$ for hidden representations after embedding.

\subsection{Linear Operator View of Sequence Models}
Sequence models can be expressed as linear operators acting on past representations. In this view, the output at time $t$ is
\begin{equation}
o_t = \sum_{i=1}^t x_i M_{t,i} = \sum_{i=1}^t x_i W_v \; g(\{x_j\}_{j=1}^t, t,i) \; W_o
\label{eq:linear-op}
\end{equation}
where $x_i \in \mathbb{R}^{d_x}$ is the input representation at position $i$, $W_v, W_o \in \mathbb{R}^{d_x \times d}$ are value and output projections, and $g(\{x_j\}_{j=1}^t,t,i) \in \mathbb{R}^{d\times d}$ is a position-dependent mixing operator that determines how position $i$ contributes to the output at position $t$. The effective weight matrix is $M_{t,i}=W_vg(\{x_j\}^t_{j=1},t,i)W_o$.

This formulation encompasses attention mechanisms, linear recurrent networks, and state-space models. The essential difference between these architectures lies in the structure of the mixing function $g$: how it aggregates information from the history $\{x_j\}_{j=1}^t$, whether it depends on content or only position, and whether it factors into separate alignment and scaling components.

\subsection{Order-Independent Aggregation: Softmax Attention}

Standard Transformer attention defines $g$ via softmax normalization for each head:
\begin{equation}
g^{\mathrm{soft}}(\{x_j\}, t,i) = \frac{\exp(q_t^\top k_i)}{\sum_{j=1}^{t} \exp(q_t^\top k_j)} \cdot I,
\label{eq:softmax}
\end{equation}
where $q_t = W_q x_t$, $k_i = W_k x_i$, and $I \in \mathbb{R}^{d \times d}$ is the identity matrix \cite{vaswani2017attention}. The scalar attention weight multiplies $I$ because standard attention applies uniform scaling across dimensions---a constraint that CAPS relaxes. This  induces order-independent contrast, since relative weights between positions depend only on content similarity to $q_t$, not temporal distance or intervening states. While suitable for global aggregation, softmax attention lacks intrinsic temporal decay, even with advanced positional embeddings \cite{su2024roformer}. Causal structure typically requires auxiliary mechanisms \cite{wu2021autoformer,nie2023a,wu2023timesnet} outside the attention kernel.

\subsection{Order-Dependent Propagation: Prefix Products}

Recurrent and state-space models induce order-dependent weighting through multiplicative dynamics. A linear recurrence $h_t = S_t h_{t-1} + B x_t$ unrolls to
\begin{equation}
o_t = \sum_{i=1}^{t} \left( \prod_{k=i+1}^{t} S_k \right) B x_i,
\label{eq:prefix-product}
\end{equation}
yielding the mixing operator
\begin{equation}
g^{\mathrm{prefix}}(\{x_j\}, t,i) = \prod_{k=i+1}^{t} S_k,
\label{eq:prefix-g}
\end{equation}
where $S_k \in \mathbb{R}^{d \times d}$ are transition operators. This structure depends explicitly on all intervening states in $(i,  t]$ and underlies recent linear-time models \citep{gu2022efficiently, gu2024mamba}. It naturally encodes causal attenuation, since an impulse at position $i$ decays through the accumulated product of transitions, with the decay rate depending only on temporal distance.

However, prefix-product operators often lack content-based selection. They determine how strongly information persists but not which information to retain. Unlike softmax attention, $g^{\mathrm{prefix}}$ cannot weight observations by relevance to the current query, limiting its ability to model global structure that requires selecting important observations across the entire history.

\subsection{Group-Theoretic Decomposition: SO(2) and SPD}
\label{sec:so2-spd}
We decompose the mixing operator $g$ using two matrix groups that separate alignment from scaling.

\begin{definition}[SO(2)]

The special orthogonal group $\mathrm{SO}(2)$ consists of $2 \times 2$ rotation matrices:
\begin{equation}
\mathrm{SO}(2) = \left\{ R_\theta = \begin{bmatrix} \cos\theta & -\sin\theta \\ \sin\theta & \cos\theta \end{bmatrix} : \theta \in [0, 2\pi) \right\}.
\end{equation}
These satisfy $R_\theta^\top R_\theta = I$ and $\det(R_\theta) = 1$. Rotations compose additively: $R_\alpha R_\beta = R_{\alpha + \beta}$.
\end{definition}

\begin{definition}[SPD]
The symmetric positive definite cone $\mathrm{SPD}(d)$ consists of matrices $\Lambda \in \mathbb{R}^{d \times d}$ satisfying $\Lambda = \Lambda^\top$ and $z^\top \Lambda z > 0$ for all $z \neq 0$. We restrict to diagonal $\mathrm{SPD}(d)$: $\Lambda = \mathrm{diag}(\lambda_1, \ldots, \lambda_d)$ with $\lambda_i > 0$.
\end{definition}

$\mathrm{SO}(2)$ preserves inner products ($\|R_\theta x\| = \|x\|$) while diagonal $\mathrm{SPD}$ scales magnitudes. This motivates factoring the mixing operator as $g(\cdot, t, i) = R_t^\top \Lambda_{t,i} R_i$, where block-diagonal $R \in \mathrm{SO}(2)^{d/2}$ handles phase alignment and diagonal $\Lambda \in \mathrm{SPD}(d)$ handles temporal scaling. CAPS implements this factorization through three additive paths for $\Lambda$.

\subsection{Motivating Example: Global---Seasonal---Local Decomposition}
\label{sec:motivating-example}

We illustrate the limitations of existing attention mechanisms through a canonical decomposition from classical time series analysis. Consider a signal composed of three components:
\begin{equation}
y_t = \underbrace{\mu_0 + \beta t}_{\text{trend}} + \underbrace{a \cos(\omega t + \varphi)}_{\text{seasonal}} + \underbrace{\textstyle\sum_{i=1}^{t} \rho^{t-i} u_i}_{\text{peak}},
\label{eq:pts-decomp}
\end{equation}
where $\beta \in \mathbb{R}$, $\omega > 0$, $\rho \in (0,1)$, and $\{u_i\}$ are sparse impulses (Figure~\ref{fig:trend-variance}).

\begin{figure}[ht]
  \begin{center}
    \centerline{\includegraphics[width=\columnwidth,keepaspectratio]{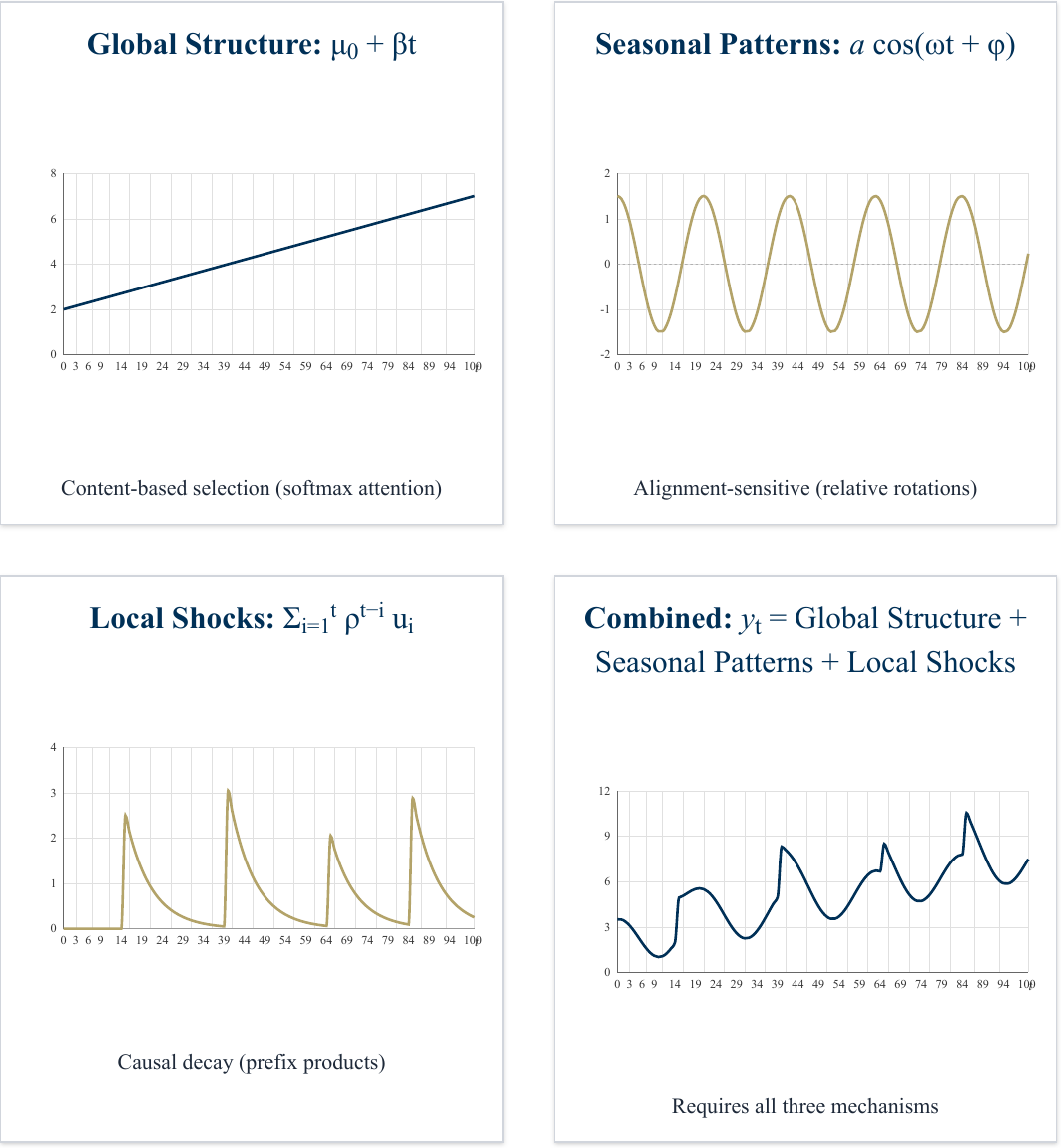}}
    \caption{Decomposition of $y_t$ into global, phase-dependent, and local components. Each requires distinct aggregation properties from the mixing operator $g$.}
    \label{fig:trend-variance}
  \end{center}
\end{figure}

Each component imposes orthogonal requirements on $g(\{x_j\}_{j=1}^t, t, i)$:
\begin{enumerate}[label=(\roman*)]
\item \textbf{Global Structure:} Estimating $\beta$ requires weighting observations by relevance, not temporal order, demanding content-based selection as in softmax attention. For instance, statistical trend estimation depends on a weighted slope, which softmax attention can recover through content-based weighting.
\item \textbf{Seasonal Patterns:} Correlation between $s_t$ and $s_i$ depends on phase offset $(t-i) \bmod \frac{2\pi}{\omega}$, requiring alignment-sensitive comparison via relative rotations.
\item \textbf{Local Shocks:} An impulse at $i^\star$ contributes $\rho^{t-i^\star} u_{i^\star}$ independently of all other tokens, requiring causal decay as in prefix products.
\end{enumerate}

\paragraph{Real-world example.} This structure is common in real-world time series data such as electricity consumption. The Electricity dataset we adopt is one such case, as long-term demand evolves smoothly due to demographic and economic factors, strong daily and weekly usage patterns induce phase-aligned seasonality, and short-lived events such as outages or behavioral shifts introduce transient shocks whose influence decays causally over time \cite{zhou2021informer}. Section \ref{sec:dataset-descriptions} further discusses our experimental datasets.

\paragraph{Why neither mechanism suffices alone.}
Softmax attention (Eq.~\ref{eq:softmax}) normalizes globally: the weight $\alpha_{t,i^\star} = \exp(q_t^\top k_{i^\star})/\sum_j \exp(q_t^\top k_j)$ depends on all tokens through the denominator, preventing token-independent exponential decay. Conversely, prefix products (Eq.~\ref{eq:prefix-g}) provide causal decay but depend only on temporal distance, precluding content-based trend estimation or alignment-sensitive comparison.

Our method resolves this by decomposing $g$ into three additive components: a Riemann softmax for order-independent aggregation, diagonal SPD prefix products for causal decay, and block-diagonal SO(2) rotations for alignment. Section \ref{sec:theoretical} formalizes this separation.

\section{Methodology}
\label{sec:methodology}

We introduce CAPS, a structured attention mechanism that decouples temporal alignment from temporal scaling. The mixing operator $g(\{x_j\}_{j=1}^t, t, i)$ in Eq.~\ref{eq:linear-op} factors as $R_t^\top \Lambda_{t,i} R_i$ following the decomposition in Section~\ref{sec:so2-spd}. CAPS implements this factorization through block-diagonal $\mathrm{SO}(2)^{d/2}$ rotations and three additive paths for the diagonal SPD scaling $\Lambda_{t,i}$, each addressing a component from Section~\ref{sec:motivating-example}. A learned Clock $\Delta_t > 0$ provides input-dependent temporal weights that couple these paths through a shared notion of temporal importance.

\subsection{CAPS Attention Layer}

\begin{figure*}[t]
  \centering
  \begin{subfigure}{0.59\textwidth}
    \centering
    \includegraphics[width=\textwidth]{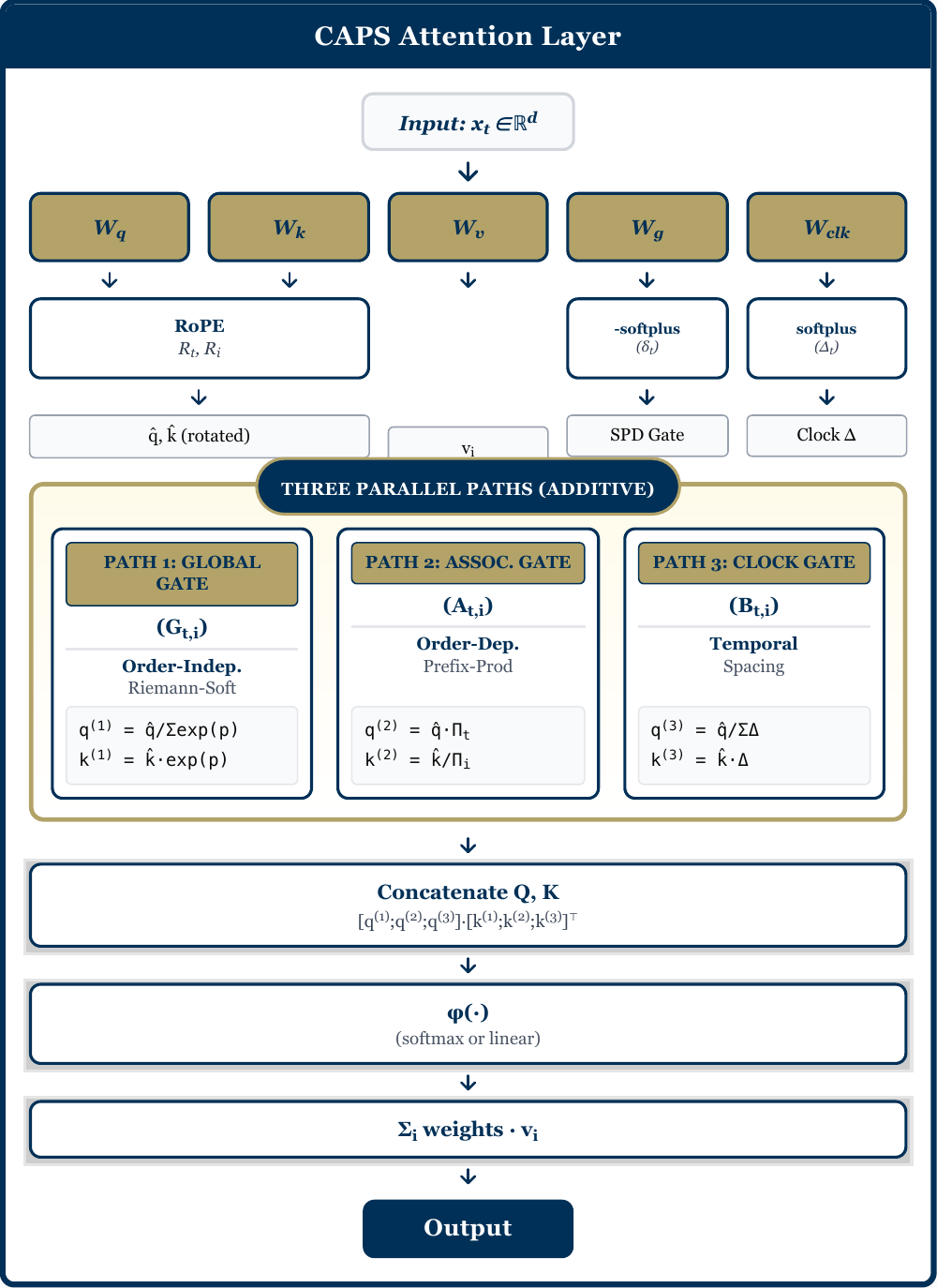}
    \caption{CAPS attention layer.}
    \label{fig:caps-attn}
  \end{subfigure}\hfill
  \begin{subfigure}{0.365\textwidth}
    \centering
    \includegraphics[width=\textwidth]{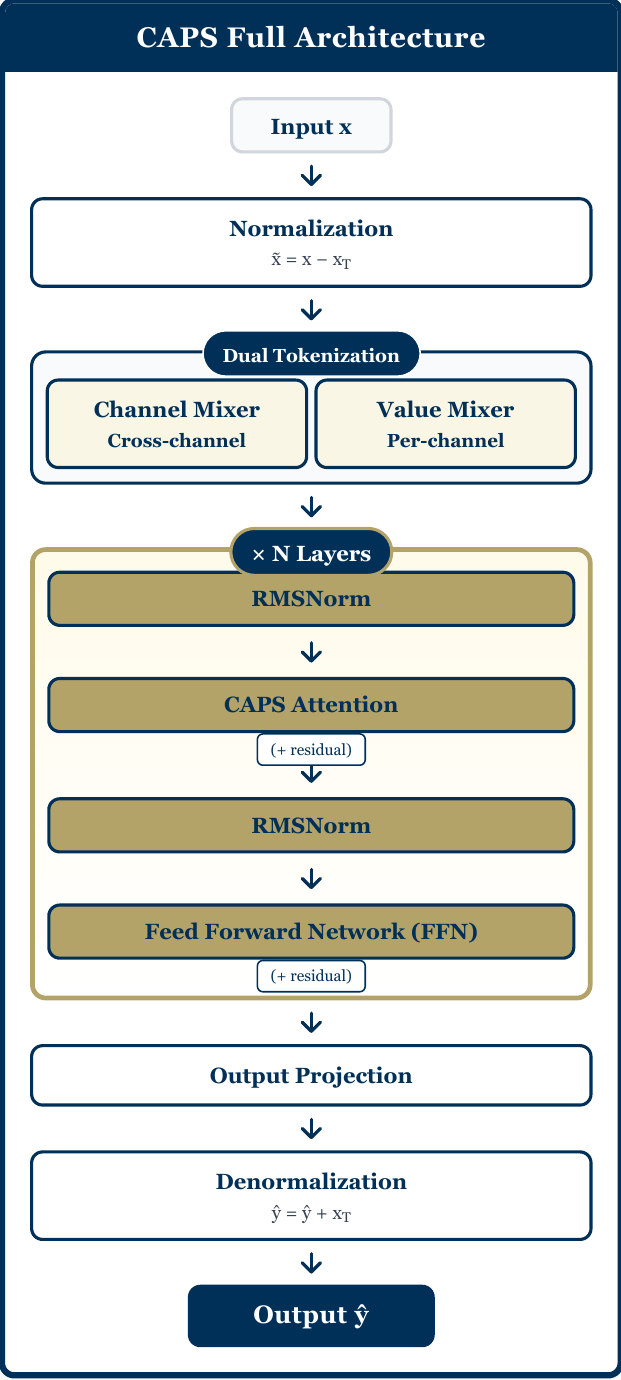}
    \caption{CAPS full architecture.}
    \label{fig:caps-arch}
  \end{subfigure}
  \caption{Overview of CAPS: (a) attention kernel, (b) full architecture.}
  \label{fig:caps-overview}
\end{figure*}

\paragraph{The Clock Mechanism.}

We introduce a learned Clock $\Delta_t \in \mathbb{R}^H_{>0}$ that assigns input-dependent temporal weights to each position:
\begin{equation}
\Delta_t = \mathrm{softplus}(W_c h_t) + \epsilon.
\label{eq:clock}
\end{equation}
where $W_c \in \mathbb{R}^{d \times H}$ is a learned projection (one output per head) and $\epsilon > 0$ is a small constant ensuring positivity. The Clock serves two roles. First, it provides a non-uniform temporal measure: intervals with larger $\Delta$ contribute more to cumulative quantities. Second, it couples the three scaling paths through a shared notion of temporal importance, enabling coherent multi-scale modeling. Each of $H$ attention heads receives independent Clock weights.

\paragraph{Temporal Alignment.}

Following RoPE \citep{su2024roformer}, each position $t$ receives rotation $R_t = \mathrm{blockdiag}(R_{t,1}, \ldots, R_{t,d/2})$ with
\begin{equation}
R_{t,\ell} = \begin{bmatrix} \cos(t\omega_\ell) & -\sin(t\omega_\ell) \\ \sin(t\omega_\ell) & \cos(t\omega_\ell) \end{bmatrix}
\end{equation}
and learned frequencies $\{\omega_\ell\}$. Rotated queries and keys $\hat{q}_t = R_t q_t$, $\hat{k}_i = R_i k_i$ satisfy $\hat{q}_t^\top \hat{k}_i = q_t^\top R_t^\top R_i k_i$, depending only on offset $(t-i)$. This captures phase relationships in the seasonal component: observations at the same cycle phase align regardless of absolute position.

We now decompose the SPD scaling $\Lambda_{t,i}$ into three additive components. Let $p_t = W_p h_t$ and $g_t = W_g h_t$ be learned gating signals.

\paragraph{Path 1: Riemann Softmax.}
For the trend component $\mu_t$, we require order-independent aggregation weighted by temporal importance. Define Clock-weighted scores $\tilde{p}_i = p_i + \log \Delta_i$ and the normalizing sum
\begin{equation}
E_t = \sum_{j=1}^t \exp(\tilde{p}_j - \tilde{p}_{\max}), \quad \tilde{p}_{\max} = \max_{j \le t} \tilde{p}_j.
\end{equation}
The first path transforms queries and keys as
\begin{equation}
q^{(1)}_t = \frac{\hat{q}_t}{E_t}, \qquad k^{(1)}_i = \exp(\tilde{p}_i - \tilde{p}_{\max}) \cdot \hat{k}_i.
\label{eq:path1}
\end{equation}
The induced weight $q^{(1)\top}_t k^{(1)}_i \propto \hat{q}_t^\top \hat{k}_i \cdot G_{t,i}$ recovers a \emph{Riemann softmax}:
\begin{equation}
G_{t,i} = \frac{\exp(p_i) \Delta_i}{\sum_{j=1}^t \exp(p_j) \Delta_j}.
\end{equation}
This generalizes standard softmax ($\Delta_i = 1$) to Clock-weighted aggregation.

\paragraph{Path 2: Prefix Product.}
For modeling local shocks, we require order-dependent decay, where an impulse at position $i$ contributes $\propto \rho^{t-i}$ independent of other tokens. This rules out any globally-normalized mechanism.

Let $\tilde{g}_j = -\mathrm{softplus}(g_j) \cdot \Delta_j < 0$ be Clock-weighted decay rates. Define the cumulative gate
\begin{equation}
\Gamma_t = \exp\left(\sum_{j=1}^t \tilde{g}_j\right).
\end{equation}
The second path transforms queries and keys as
\begin{equation}
q^{(2)}_t = \Gamma_t \cdot \hat{q}_t, \qquad k^{(2)}_i = \frac{\hat{k}_i}{\Gamma_i}.
\label{eq:path2}
\end{equation}f
The induced weight satisfies
\begin{equation}
q^{(2)\top}_t k^{(2)}_i = \hat{q}_t^\top \hat{k}_i \cdot \frac{\Gamma_t}{\Gamma_i} = \hat{q}_t^\top \hat{k}_i \cdot \exp\left(\sum_{j=i+1}^t \tilde{g}_j\right) = \hat{q}_t^\top \hat{k}_i \cdot A_{t,i}.
\end{equation}
The decay $A_{t,i}$ depends only on the interval $(i, t]$---it is decoupled from all tokens outside this interval. This resolves the fundamental limitation of softmax identified in Section~\ref{sec:motivating-example}.

\paragraph{Path 3: Clock Baseline.}
The third path provides content-agnostic temporal weighting:
\begin{equation}
q^{(3)}_t = \frac{\hat{q}_t}{\sum_{j=1}^t \Delta_j}, \qquad k^{(3)}_i = \Delta_i \cdot \hat{k}_i.
\label{eq:path3}
\end{equation}
The induced weight $q^{(3)\top}_t k^{(3)}_i \propto \hat{q}_t^\top \hat{k}_i \cdot \Delta_i / \sum_j \Delta_j$ encodes a learned prior over temporal importance, independent of query-key content.

\paragraph{Additive Composition.}

The three paths combine via concatenation:
\begin{equation}
q_{\mathrm{cat}} = [q^{(1)}; q^{(2)}; q^{(3)}], \qquad k_{\mathrm{cat}} = [k^{(1)}; k^{(2)}; k^{(3)}].
\end{equation}
The attention score decomposes additively:
\begin{equation}
\alpha_{t,i} \propto q_{\mathrm{cat},t}^\top k_{\mathrm{cat},i} = \underbrace{\vphantom{\sum} q^{(1)\top}_t k^{(1)}_i}_{\text{global } G_{t,i}} + \underbrace{\vphantom{\sum} q^{(2)\top}_t k^{(2)}_i}_{\text{causal } A_{t,i}} + \underbrace{\vphantom{\sum} q^{(3)\top}_t k^{(3)}_i}_{\text{baseline } B_{t,i}},
\label{eq:additive}
\end{equation}
with output $o_t = \sum_{i=1}^t \phi(\alpha_{t,i}) v_i$. Setting $\phi = \mathrm{softmax}$ yields quadratic complexity; setting $\phi = \mathrm{identity}$ (no normalization) yields linear complexity (Proposition~\ref{prop:complexity}). We report results with the linear variant in Table~\ref{tab:summarized-results}.

Each path addresses a component from Eq.~\ref{eq:pts-decomp}: Path 1 (Riemann softmax) captures trends via order-independent selection; Path 2 (prefix product)  captures peaks via token-independent causal decay; Path 3 (clock baseline) provides a temporal prior; RoPE captures seasonality via phase-sensitive alignment. The model learns appropriate combinations without manual specification.

\subsection{Architecture Summary}
The full architecture (as seen in \cref{fig:caps-arch}) applies last-value normalization, extends the input sequence to length $L+H$ via a learned linear layer, and uses dual tokenization combining cross-channel context with per-channel value embeddings. Following \citet{liu2024itransformer}, temporal processing operates independently per channel. We apply random-ratio channel dropout during training~\citep{lu2024cats}. Full details are provided in ~\cref{sec:architecture-details}.

\subsection{Complexity}

\begin{proposition}[Complexity]
\label{prop:complexity}
CAPS with linear attention has complexity $O(Td^2)$; with softmax attention, $O(T^2d)$. The three-path concatenation increases constants by factor $3$ without changing asymptotics.
\end{proposition}

\begin{proof}
Computing $E_t$, $\Gamma_t$, and $\sum_{j=1}^{t}\Delta_j$ each requires $O(T)$ cumulative sums per head. Concatenation yields $q_{\mathrm{cat}}, k_{\mathrm{cat}} \in \mathbb{R}^{T \times H \times 3d_h}$ where $d_h = d/H$. Linear attention computes $\sum_i k_i v_i^\top \in \mathbb{R}^{3d_h \times d_h}$ in $O(T \cdot d_h^2)$ per head, totaling $O(T d^2 / H)$. Since $H$ is constant, this is $O(Td^2)$. Softmax computes $T \times T$ scores in $O(T^2 d_h)$ per head, totaling $O(T^2 d)$.
\end{proof}

\subsection{Theoretical Properties}
\label{sec:theoretical}

We formalize the limitations of softmax attention and establish that CAPS overcomes them.

\begin{proposition}[Softmax Couples All Tokens]
\label{prop:coupling}
Let $\alpha_{t,i} = \exp(s_{t,i})/\sum_{j=1}^{t} \exp(s_{t,j})$ be causal softmax attention. For any $i^\star < t$ and $j \neq i^\star$:
\[
\frac{\partial \alpha_{t,i^\star}}{\partial s_{t,j}} = -\alpha_{t,i^\star} \cdot \alpha_{t,j} \neq 0.
\]
Consequently, no softmax attention achieves token-independent decay $\alpha_{t,i^\star} = \rho^{t-i^\star}$.
\end{proposition}

\begin{proof}
Direct computation gives $\partial \alpha_{t,i^\star}/\partial s_{t,j} = -\exp(s_{t,i^\star})\exp(s_{t,j})/Z_t^2 = -\alpha_{t,i^\star}\alpha_{t,j} \neq 0$. If $\alpha_{t,i^\star} = \rho^{t-i^\star}$ held for all inputs, perturbing $s_{t,j}$ via $x_j$ would change $\alpha_{t,i^\star}$, a contradiction.
\end{proof}

\begin{proposition}[CAPS Achieves Decoupled Aggregation]
\label{prop:caps-achieves}
A single CAPS layer simultaneously provides:
\begin{enumerate}[label=(\roman*)]
    \item \textbf{Content-based selection:} Path~1 weights $G_{t,i} \propto \exp(p_i)\Delta_i$ have numerators depending only on position $i$, not on distance $(t-i)$.
    \item \textbf{Token-independent decay:} Path~2 weights $A_{t,i} = \exp\bigl(\sum_{j=i+1}^{t} \tilde{g}_j\bigr)$ depend only on the interval $(i,t]$.
    \item \textbf{Phase-sensitive alignment:} RoPE yields $\hat{q}_t^\top \hat{k}_i = q_t^\top R_{i-t} k_i$, depending only on offset $(i-t)$.
\end{enumerate}
\end{proposition}

\begin{proof}
(i) By construction, $G_{t,i} = \exp(p_i)\Delta_i / \sum_{j=1}^{t}\exp(p_j)\Delta_j$; the numerator is local to position $i$.
(ii) $A_{t,i} = \Gamma_t / \Gamma_i = \exp(\sum_{j=i+1}^{t}\tilde{g}_j)$; tokens at positions $\ell \leq i$ do not appear.
(iii) Since $R_\theta^\top = R_{-\theta}$ and rotations compose additively, $R_t^\top R_i = R_{-t}R_i = R_{i-t}$.
\end{proof}

These three properties correspond precisely to the requirements identified in Section~\ref{sec:motivating-example}: (i)~for trend estimation, (ii)~for transient shocks, and (iii)~for seasonal alignment.

\section{Experiments}
\label{sec:experiments}

\begin{table*}[t]
  \caption{Summarized forecasting results (MSE). Results are averaged over $H \in \{96, 192, 336, 720\}$ for long-term datasets and $H \in \{12, 24, 48, 96\}$ for PEMS. \textbf{Bold}: best, \underline{underline}: second best. Full results are provided in \cref{tab:full-results}.}
  \label{tab:summarized-results}
  \begin{center}
    \begin{small}
      \setlength{\tabcolsep}{4pt}
      \resizebox{\textwidth}{!}{
      \begin{tabular}{c|c|c|c||c|c|c|c|c|c|c}
        \toprule
        \multicolumn{2}{c|}{\textbf{Model}} & \makecell{\textbf{CAPS} \\ \scriptsize (Ours)} & \makecell{\textbf{LinAttn+RoPE} \\ \scriptsize (Baseline)} & \makecell{\textbf{OLinear} \\ \scriptsize \citep{yue2025olinear}} & \makecell{\textbf{TimeMixer++} \\ \scriptsize \citep{wang2025timemixer}} & \makecell{\textbf{TimeMixer} \\ \scriptsize \citep{wang2024timemixer}} & \makecell{\textbf{iTransformer} \\ \scriptsize \citep{liu2024itransformer}} & \makecell{\textbf{PatchTST} \\ \scriptsize \citep{nie2023a}} & \makecell{\textbf{TimesNet} \\ \scriptsize \citep{wu2023timesnet}} & \makecell{\textbf{DLinear} \\ \scriptsize \citep{zeng2023transformers}} \\
        \midrule
        \multirow{7}{*}{\rotatebox{90}{Long-Term}} & Weather & \underline{0.235} & 0.253 & 0.237 & \textbf{0.226} & 0.240 & 0.258 & 0.265 & 0.259 & 0.265 \\
        & Solar & \textbf{0.198} & 0.215 & 0.215 & \underline{0.203} & 0.216 & 0.233 & 0.287 & 0.403 & 0.330 \\
        & Electricity & 0.177 & 0.183 & \textbf{0.159} & \underline{0.165} & 0.182 & 0.178 & 0.216 & 0.193 & 0.225 \\
        & ETTh1 & 0.425 & 0.463 & \underline{0.424} & \textbf{0.419} & 0.447 & 0.454 & 0.507 & 0.458 & 0.461 \\
        & ETTh2 & 0.387 & 0.414 & 0.367 & \textbf{0.339} & \underline{0.365} & 0.383 & 0.391 & 0.414 & 0.563 \\
        & ETTm1 & \textbf{0.368} & 0.387 & 0.375 & \underline{0.369} & 0.381 & 0.410 & 0.402 & 0.400 & 0.404 \\
        & ETTm2 & 0.277 & 0.282 & \underline{0.270} & \textbf{0.269} & 0.275 & 0.288 & 0.290 & 0.291 & 0.354 \\
        \midrule
        \multirow{3}{*}{\rotatebox{90}{Short}} & PEMS03 & \textbf{0.093} & 0.127 & \underline{0.096} & 0.165 & 0.167 & 0.113 & 0.180 & 0.147 & 0.278 \\
        & PEMS04 & \textbf{0.087} & 0.093 & \underline{0.091} & 0.136 & 0.185 & 0.111 & 0.195 & 0.129 & 0.295 \\
        & PEMS08 & \underline{0.115} & 0.136 & \textbf{0.113} & 0.201 & 0.226 & 0.150 & 0.280 & 0.193 & 0.379 \\
        \midrule
        \multicolumn{2}{c|}{Avg Rank} & \textbf{2.3} & 5.0 & \textbf{2.3} & 2.8 & 4.8 & 5.1 & 7.7 & 6.5 & 8.6 \\
        \multicolumn{2}{c|}{Top-1 Count} & \textbf{4} & 0 & 2 & \textbf{4} & 0 & 0 & 0 & 0 & 0 \\
        \bottomrule
      \end{tabular}
      }
    \end{small}
  \end{center}
\end{table*}

\subsection{Experimental Setup}

\paragraph{Datasets.} We train and evaluate on ten widely used real-world multivariate time series datasets spanning weather, energy, electricity load, transformer monitoring, and traffic forecasting domains, and covering both long-term and short-term forecasting settings. Section \ref{sec:dataset-descriptions} further describes these datasets and our rationale for selecting each.

\paragraph{Baselines.} We evaluate against seven strong baselines:  \textbf{Olinear} \cite{yue2025olinear}, \textbf{TimeMixer++} and \textbf{TimeMixer} \citep{wang2025timemixer}, \textbf{iTransformer} \citep{liu2024itransformer}, \textbf{PatchTST} \citep{nie2023a}, \textbf{TimesNet} \citep{wu2023timesnet}, and \textbf{DLinear} \cite{zeng2023transformers}. 

To isolate the contribution of our three-path gating mechanism, we also compare \textbf{CAPS (Linear)} against linear attention~\cite{katharopoulos2020transformers} with RoPE~\cite{su2024roformer} (denoted \textbf{LinAttn+RoPE}).  We present these results in Table \ref{tab:summarized-results}. In our ablation studies, we also compare \textbf{CAPS (Softmax)} against vanilla softmax attention \cite{vaswani2017attention} with RoPE \cite{su2024roformer} (denoted \textbf{Attn+RoPE}). 

We also provide comparison to Informer \cite{zhou2021informer} and Autoformer \cite{wu2021autoformer} in addition to many of our baselines in Table \ref{tab:efficiency-full}

\paragraph{Fair Comparison.} We use the widely used Time-Series-Library codebase\footnote{\url{https://github.com/thuml/Time-Series-Library}} \citep{wu2021autoformer,wu2023timesnet} and evaluate all methods under the same environment (see \cref{tab:summarized-results,tab:full-results,tab:ablation-main,tab:efficiency-full,tab:softmax-vs-linear,tab:ablation,tab:attention-ablation-full}). We note two exceptions. Because we  could not reproduce the results of \citet{wang2025timemixer}, within our compute budget, we source the results of TimeMixer++ from their paper. These numbers favor TimeMixer++. We also source the results of OLinear from their paper  \cite{yue2025olinear}.

\paragraph{Hyperparameter and Implementation Details.} In our main results (\cref{tab:summarized-results}) we report Mean Squared Error (MSE) to avoid bias, as we use MSE loss. We discuss all training and implementation details in \cref{sec:impl-details}. 

\subsection{Main Results}
\label{sec:main-exp}

\paragraph{Overall Performance.}
Table~\ref{tab:summarized-results} summarizes forecasting performance across ten datasets.
CAPS achieves an average rank of 2.3 with four first-place results, simultaneously matching the best average-rank performance (OLinear) and the best top-1 count (TimeMixer++). No other method attains best (tied) values on both metrics, indicating consistently strong performance across diverse temporal structures. In contrast, competing methods achieve strong performance on specific datasets but exhibit higher variance across benchmarks.

\paragraph{Isolating the Three-Path Mechanism.}
The comparison against LinAttn+RoPE directly measures the contribution of the three-path decomposition, as both models share identical backbone architecture and differ only in the attention kernel. CAPS outperforms this baseline on all ten datasets. On long-term benchmarks, MSE reductions range from 1.8\% (ETTm2) to 8.2\% (ETTh1); on short-term traffic data, improvements reach 6.5\% (PEMS04) to 26.8\% (PEMS03). These consistent gains confirm that the Riemann, prefix-product, and Clock paths drive performance improvements.

\paragraph{Long-term Forecasting.}
Across seven long-term datasets, CAPS achieves first place on Solar and ETTm1, and second place on Weather. TimeMixer++ and OLinear lead on the remaining datasets through frequency-domain decomposition and optimized linear projections respectively. These methods address temporal structure prior to or outside the mixing operation, while CAPS modifies the mixing operator itself. The approaches are complementary: CAPS attention could serve as the mixing mechanism within a multi-scale architecture, a direction we discuss in Section~\ref{sec:conclusion}.

\paragraph{Short-term Forecasting.}
On the PEMS traffic benchmarks, CAPS achieves the lowest MSE on PEMS03 and PEMS04, and second-lowest on PEMS08. Traffic data exhibits rapid transient dynamics where the prefix-product path provides appropriate inductive bias through token-independent causal decay. Frequency-domain methods show notably weaker performance: TimeMixer++ achieves 0.165 MSE on PEMS03 compared to 0.093 for CAPS.

\FloatBarrier
\subsection{Ablation Studies}

We ablate two design choices: the attention normalization (softmax vs.\ linear) and each gating path within the three-path kernel.

\paragraph{Softmax vs. Linear Attention.}
\begin{table}[t]
  \caption{Attention mechanism comparison (MSE averaged over $H \in \{96, 192, 336, 720\}$). Full results are provided in Table \ref{tab:attention-ablation-full}.}
  \label{tab:softmax-vs-linear}
  \begin{center}
    \begin{small}
      \setlength{\tabcolsep}{4pt}
      \resizebox{\columnwidth}{!}{
      \begin{tabular}{l|c|c|c|c|c}
        \toprule
        \textbf{Model} & \makecell{\textbf{CAPS} \\ \scriptsize (Linear)} & \makecell{\textbf{CAPS} \\ \scriptsize (Softmax)} & \makecell{\textbf{Attn} \\ \scriptsize +RoPE} & \makecell{\textbf{Pure} \\ \scriptsize Softmax} & \makecell{\textbf{LinAttn} \\ \scriptsize +RoPE} \\
        \midrule
        Weather & \textbf{0.235} & 0.253 & 0.244 & \underline{0.243} & 0.253 \\
        Solar & \textbf{0.198} & 0.212 & 0.214 & \underline{0.211} & 0.215 \\
        Electricity & \textbf{0.177} & 0.184 & \underline{0.179} & 0.182 & 0.183 \\
        ETTh1 & \textbf{0.425} & 0.441 & 0.449 & \underline{0.438} & 0.463 \\
        ETTh2 & \textbf{0.387} & \underline{0.400} & 0.401 & 0.415 & 0.414 \\
        ETTm1 & \textbf{0.368} & 0.389 & 0.379 & \underline{0.377} & 0.387 \\
        ETTm2 & \textbf{0.277} & 0.285 & 0.287 & 0.283 & \underline{0.282} \\
        \midrule
        Average & \textbf{0.295} & 0.309 & \underline{0.307} & \underline{0.307} & 0.314 \\
        \bottomrule
      \end{tabular}
      }
    \end{small}
  \end{center}
  \vspace{-1em}
\end{table}

Table~\ref{tab:softmax-vs-linear} compares CAPS under both attention variants. CAPS (Linear) outperforms CAPS (Softmax) on all seven datasets, with average MSE of 0.295 vs.\ 0.309. CAPS is built for linear attention. Each path incorporates its own normalization: Riemann softmax in Path 1, prefix products in Path 2, Clock weighting in Path 3, enabling the paths to operate independently. Wrapping CAPS in softmax introduces redundant normalization that re-couples the paths through a global denominator, negating the decoupling we design. This explains the asymmetry in Table~\ref{tab:softmax-vs-linear}: adding the three-path mechanism to linear attention reduces MSE by 6.1\% (LinAttn+RoPE 0.314 $\to$ CAPS 0.295), while adding it to softmax offers no performance gains overall (Pure Softmax 0.307 $\to$ CAPS Softmax 0.309).

\paragraph{Path Ablation.}

\begin{table}[t]
  \caption{Ablation study on CAPS components (MSE averaged over $H \in \{96, 192, 336, 720\}$). Full results are provided in Table \ref{tab:ablation}.}
  \label{tab:ablation-main}
  \begin{center}
    \begin{small}
      \setlength{\tabcolsep}{4pt}
      \resizebox{\columnwidth}{!}{
      \begin{tabular}{l|c|c|c|c}
        \toprule
        \textbf{Model} & \makecell{\textbf{CAPS} \\ \scriptsize (Full)} & \makecell{\textbf{w/o Path 1} \\ \scriptsize (Riemann)} & \makecell{\textbf{w/o Path 2} \\ \scriptsize (Prefix)} & \makecell{\textbf{w/o Path 3} \\ \scriptsize (Clock)} \\
        \midrule
        Weather & \textbf{0.235} & 0.246 & 0.247 & \underline{0.244} \\
        ETTh1 & \textbf{0.425} & \underline{0.426} & \underline{0.426} & 0.430 \\
        ETTh2 & \textbf{0.387} & \underline{0.389} & 0.390 & \underline{0.389} \\
        ETTm1 & \textbf{0.368} & \underline{0.371} & 0.373 & 0.373 \\
        ETTm2 & \textbf{0.277} & \underline{0.280} & 0.282 & 0.281 \\
        \midrule
        Avg $\Delta$ & -- & -1.17\% & -1.54\% & -1.46\% \\
        \bottomrule
      \end{tabular}
      }
    \end{small}
  \end{center}
  \vspace{-1em}
\end{table}

Table~\ref{tab:ablation-main} removes each path individually, demonstrating that all three paths contribute meaningfully. Average MSE degrades by 1.2\% (Path 1), 1.5\% (Path 2), and 1.5\% (Path 3) upon removal. The prefix-product path contributes most overall, consistent with its role in modeling transient dynamics. The Riemann path matters most on Weather, which exhibits smooth trends requiring order-independent aggregation, and the prefix-product path matters most on ETTm2, where short-horizon targets amplify the importance of local decay.

\FloatBarrier
\subsection{Computational Cost}

\begin{table}[t]
  \caption{Computational efficiency comparison on ETTm1 with input length $L=96$. FLOPs measures forward-pass computation (K=$10^3$, M=$10^6$, G=$10^9$); Params denotes learnable parameters.}
  \label{tab:efficiency-full}
  \begin{center}
    \begin{small}
      \setlength{\tabcolsep}{3pt}
      \resizebox{\columnwidth}{!}{
      \begin{tabular}{l|cc|cc|cc|cc}
        \toprule
        \multirow{2}{*}{\textbf{Model}} & \multicolumn{2}{c|}{$H=96$} & \multicolumn{2}{c|}{$H=192$} & \multicolumn{2}{c|}{$H=336$} & \multicolumn{2}{c}{$H=720$} \\
        & FLOPs & Params & FLOPs & Params & FLOPs & Params & FLOPs & Params \\
        \midrule
        CAPS \scriptsize{(Ours)} & 1.39G & 527K & 2.09G & 537K & 3.13G & 550K & 5.92G & 587K \\
        TimeMixer++ & 6.24G & 1.19M & 6.25G & 1.20M & 6.25G & 1.22M & 6.26G & 1.27M \\
        TimeMixer & 2.77G & 1.13M & 2.81G & 1.14M & 2.85G & 1.17M & 2.98G & 1.24M \\
        iTransformer & 210M & 9.56M & 211M & 9.61M & 213M & 9.68M & 217M & 9.88M \\
        PatchTST & 2.51G & 10.1M & 2.52G & 10.6M & 2.54G & 11.5M & 2.59G & 13.9M \\
        TimesNet & 2.26G & 1.19M & 3.38G & 1.20M & 5.06G & 1.22M & 9.57G & 1.25M \\
        DLinear & 259K & 18.6K & 517K & 37.2K & 905K & 65.2K & 1.94M & 140K \\
        Informer & 2.41G & 15.3M & 3.12G & 15.3M & 4.18G & 15.3M & 7.02G & 15.3M \\
        Autoformer & 2.98G & 13.7M & 3.69G & 13.7M & 4.76G & 13.7M & 7.60G & 13.7M \\
        \bottomrule
      \end{tabular}
      }
    \end{small}
  \end{center}
\end{table}
Table~\ref{tab:efficiency-full} reports FLOPs and parameter counts on ETTm1. CAPS requires 527K parameters, which is 18$\times$ fewer than iTransformer (9.56M) and PatchTST (10.1M), while achieving competitive accuracy. At $H=96$, CAPS uses 1.39G FLOPs, below TimeMixer (2.77G), TimesNet (2.26G), and TimeMixer++ (6.24G). Compared to DLinear (259K FLOPs, 18.6K parameters), CAPS incurs additional cost but reduces MSE by 8.9\% on ETTm1. The three-path construction increases constant factors over standard linear attention without changing asymptotic complexity.

\FloatBarrier
\section{Conclusion}
\label{sec:conclusion}

 We introduced CAPS, a structured attention mechanism that decouples global aggregation, causal decay, and phase alignment through three additive gating paths unified by a learned Clock. We observe that softmax attention couples all tokens through global normalization (Proposition~\ref{prop:coupling}), preventing token-independent decay for transient dynamics. Empirically, CAPS outperforms LinAttn+RoPE on all ten datasets and achieves competitive results against strong baselines, with average rank 2.3 and four first-place finishes in summarized evaluation. Ablations confirm that CAPS is optimized for linear attention, as the three-path mechanism improves linear attention by 6.1\% but does not improve softmax attention overall. Ablations further demonstrate that all three paths contribute, with removal degrading MSE by 1.2--1.5\%. CAPS admits linear complexity with 527K parameters, 18$\times$ fewer than iTransformer and PatchTST.

CAPS and frequency-domain methods such as TimeMixer++~\citep{wang2025timemixer} address temporal structure through complementary mechanisms: TimeMixer++ decomposes signal representations before mixing, while CAPS modifies the mixing operator itself. These approaches are orthogonal, and combining them by using CAPS attention within a multi-scale architecture is a natural direction. The linear variant also offers an interpretability advantage: the additive weight decomposition $w_{t,i} = w^{(1)}_{t,i} + w^{(2)}_{t,i} + w^{(3)}_{t,i}$ allows direct inspection of each path's contribution, mirroring classical STL decomposition~\citep{cleveland1990stl}. Future work could leverage this transparency to diagnose which temporal structures dominate in different forecasting regimes. A key limitation is that we have not explored large-scale pretraining, and whether CAPS scales to foundation-model regimes remains an open question.
 
\section*{Impact Statement}

This paper contributes to the field of Machine Learning, in particular, time series forecasting. The proposed CAPS attention mechanism has broad applications including financial decision-making, energy grid management, weather prediction, and traffic planning, which can contribute to resource efficiency and public safety. We do not anticipate specific negative societal consequences arising from this work, but it is important to ensure responsible deployment and oversight, especially in safety-critical domains to mitigate any potential risks or negative outcomes.


\bibliography{example_paper}
\bibliographystyle{icml2026}

\newpage
\appendix
\onecolumn
\section{Appendix}

\subsection{Dataset Descriptions}
\label{sec:dataset-descriptions}

We adopt ten widely used real-world multivariate time series datasets for our long- and short-term forecasting experiment settings. These datasets can be summarized as follows:

\paragraph{Weather Dataset \cite{wu2021autoformer}.} The Weather dataset contains 21 meteorological indicators such as air temperature and humidity, collected at ten-minute intervals throughout 2020 by the weather station of the Max Planck Institute for Biogeochemistry in Germany.

\paragraph{Solar Dataset \cite{lai2018modeling}.} The Solar dataset consists of high-resolution solar energy generation data from 137 photovoltaic power plants located in Alabama in 2006.

\paragraph{Electricity Dataset \cite{wu2021autoformer}.} The Electricity dataset reports hourly electricity consumption for 321 customers from 2012 to 2014. 

\paragraph{ETT Dataset \cite{zhou2021informer}.} The Electricity Transformer Temperature (ETT) dataset includes seven variables related load and oil temperature data from electricity transformers. We adopt four subsets (ETTm1, ETTm2, ETTh1, and ETTh2).

\paragraph{PEMS Dataset \cite{li2017diffusion}.} The PEMS dataset includes traffic data collected by California Transportation Agencies (CalTrans) Performance Measurement Systems (PeMS) at five-minute intervals. We adopt three subsets (PEMS03, PEMS04, and PEMS08) which have different numbers of variables.

\paragraph{Dataset Selection.} 
We exclude the Exchange dataset \cite{lai2018modeling}, which records daily exchange rates across eight currencies. In efficient markets, exchange rates are well-modeled as random walks (e.g., an AR(1) process with unit root) for which a last-value baseline is the best predictor \cite{fama1970efficient, rossi2013exchange, nie2023a}. Recent empirical results further confirm that this baseline is competitive with Transformer-based models on this dataset \cite{zeng2023transformers}. For these reasons, we omit the Exchange dataset, following common practice in recent forecasting studies.

We additionally omit the large-scale Traffic \cite{wu2021autoformer} and PEMS07 \cite{li2017diffusion} datasets, as we are unable to reproduce TimeMixer++ \cite{wang2025timemixer} results for these datasets within our computational budget. To ensure fair and consistent comparison across methods, we restrict evaluation to datasets for which all models can be reliably trained and evaluated under comparable settings.

\subsection{Hyperparameter Settings and Implementation Details}
\label{sec:impl-details}

We use $N=3$ Transformer encoder layers with $H=4$ attention heads. The hidden dimension $d_{\mathrm{model}}$ is set using dataset-specific endogenous and exogenous scaling factors: $f_{\mathrm{endo}}=f_{\mathrm{exo}}=64$ for ETTm1, ETTm2, Weather, Solar, Electricity, Traffic, and all PEMS datasets, and $f_{\mathrm{endo}}=f_{\mathrm{exo}}=8$ for ETTh1 and ETTh2. We use RMSNorm \cite{zhang2019root} as the normalization layer. We initialize the weights of all linear layers using a normal distribution with mean 0 and standard deviation 0.02; for output projection layers, we additionally scale the standard deviation by $1/\sqrt{2N}$ following GPT-2 conventions. All experiments use random seed 2026 and can be conducted on a single NVIDIA RTX4090 or NVIDIA L40S GPU. The base batch size is 32; for larger datasets (Electricity, Traffic, Solar, PEMS), we use batch sizes of 2--4 with 8--16 step gradient accumulation to maintain an effective batch size of 32. We train using MSE loss with the AdamW optimizer, betas $(0.9, 0.999)$, and a OneCycleLR scheduler. Weight decay is set to $10^{-5}$ for most datasets, increased to $0.1$ for ETTh1 and ETTh2, and disabled for high-dimensional datasets. Gradient clipping is applied with max norm 1, and early stopping patience is set to 12 epochs. Random Ratio Channel Dropout is enabled for low-dimensional datasets (ETT, Weather) and disabled for high-dimensional datasets where channel diversity provides sufficient regularization. We follow the train-validation-test splitting ratio of \citet{wang2025timemixer} and adapt the experimental environment inherited from Autoformer and Time-Series-Library \footnote{\url{https://github.com/thuml/Time-Series-Library} \citep{wu2023timesnet, wu2021autoformer}}.

\subsection{Full Architecture}
\label{sec:architecture-details}

\paragraph{Last Value Normalization.}
Each channel is centered by subtracting its final observed value $x_{c,L}$: $\tilde{x}_c = x_c - x_{c,L}$. After prediction, this shift is reversed: $\hat{y}_c = \tilde{y}_c + x_{c,L}$. This simple normalization handles level shifts and non-stationarity without requiring variance estimation.

\paragraph{Horizon Extension.}
The input sequence $x \in \mathbb{R}^{B \times C \times L}$ is extended to length $L + H$ via a learned linear layer $W_{\text{ext}} \in \mathbb{R}^{L \times H}$:
\begin{equation}
x_{\text{ext}} = [x; W_{\text{ext}}(x)] \in \mathbb{R}^{B \times C \times (L+H)}.
\end{equation}
This provides a learnable initialization for the forecasting horizon that the  model refines through temporal processing.

\paragraph{Dual Tokenization.}
Each position receives two token types. Channel tokens $c_t \in \mathbb{R}^{d_{\text{model}}}$ aggregate information across all input channels via $c_t = W_c x_t$ where $x_t \in \mathbb{R}^{C_{\text{in}}}$. Value tokens $v_{c,t} \in \mathbb{R}^{d_{\text{emb}}}$ encode per-channel magnitudes: $v_{c,t} = x_{c,t} \cdot V_c$ where $V_c \in \mathbb{R}^{1 \times d_{\text{emb}}}$ is a learned embedding. The representation $h_{c,t} = [c_t; v_{c,t}] \in \mathbb{R}^{d_{\text{model}} + d_{\text{emb}}}$ combines cross-channel context with channel-specific scaling.

\paragraph{Channel-Independent Processing.}
Following \citet{liu2024itransformer}, the temporal processor operates independently on each target channel by reshaping $(B, C, T, D) \to (BC, T, D)$. Temporal mixing weights are shared across channels while preserving channel-specific dynamics. Only target channels are decoded: $\hat{y}_{c,t} = h_{c,t} V_c^{\top}$ for $c \in [1, C_{\text{target}}]$ and $t \in [L+1, L+H]$.

\paragraph{Random Ratio Channel Dropout.}
Following common practice \cite{lu2024cats, lu2023arm}, during training, each sample independently drops a random fraction of input channels. For batch element $b$, we sample $r_b \sim \text{Uniform}(0,1)$ and mask each channel with probability $r_b$, rescaling survivors by $(1-r_b)^{-1}$. This prevents the channel tokenization from over-fitting to specific channel subsets and improves robustness on high-dimensional inputs.

\subsection{Full Experimental Results}
\label{sec:full-results}

\begin{table*}[ht]
  \caption{Full experimental results. All models use $L=96$. We report MSE and MAE (lower is better). \textbf{Bold}: best, \underline{underline}: second best.}
  \label{tab:full-results}
  \begin{center}
    \begin{small}
      \setlength{\tabcolsep}{4pt}
      \resizebox{\textwidth}{!}{
      \begin{tabular}{c|c|cc|cc||cc|cc|cc|cc|cc|cc|cc}
        \toprule
        \multicolumn{2}{c|}{\textbf{Model}} & \multicolumn{2}{c|}{\makecell{\textbf{CAPS} \\ \scriptsize (Ours)}} & \multicolumn{2}{c||}{\makecell{\textbf{LinAttn+RoPE} \\ \scriptsize (Baseline)}} & \multicolumn{2}{c|}{\makecell{\textbf{OLinear} \\ \scriptsize \citep{yue2025olinear}}} & \multicolumn{2}{c|}{\makecell{\textbf{TimeMixer++} \\ \scriptsize \citep{wang2025timemixer}}} & \multicolumn{2}{c|}{\makecell{\textbf{TimeMixer} \\ \scriptsize \citep{wang2024timemixer}}} & \multicolumn{2}{c|}{\makecell{\textbf{iTransformer} \\ \scriptsize \citep{liu2024itransformer}}} & \multicolumn{2}{c|}{\makecell{\textbf{PatchTST} \\ \scriptsize \citep{nie2023a}}} & \multicolumn{2}{c|}{\makecell{\textbf{TimesNet} \\ \scriptsize \citep{wu2023timesnet}}} & \multicolumn{2}{c}{\makecell{\textbf{DLinear} \\ \scriptsize \citep{zeng2023transformers}}} \\
        \midrule
        \multicolumn{2}{c|}{\textbf{Metric}} & MSE & MAE & MSE & MAE & MSE & MAE & MSE & MAE & MSE & MAE & MSE & MAE & MSE & MAE & MSE & MAE & MSE & MAE \\
        \midrule
        \multirow{5}{*}{\rotatebox{90}{Weather}}
        & 96 & \textbf{0.144} & \underline{0.202} & \underline{0.152} & 0.211 & 0.153 & \textbf{0.190} & 0.155 & 0.205 & 0.163 & 0.209 & 0.174 & 0.214 & 0.186 & 0.227 & 0.172 & 0.220 & 0.195 & 0.252 \\
& 192 & \textbf{0.198} & 0.261 & 0.210 & 0.270 & \underline{0.200} & \textbf{0.235} & 0.201 & \underline{0.245} & 0.208 & 0.250 & 0.221 & 0.254 & 0.234 & 0.265 & 0.219 & 0.261 & 0.237 & 0.295 \\
& 336 & \underline{0.251} & 0.300 & 0.280 & 0.329 & 0.258 & \underline{0.280} & \textbf{0.237} & \textbf{0.265} & 0.251 & 0.287 & 0.278 & 0.296 & 0.284 & 0.301 & 0.280 & 0.306 & 0.282 & 0.331 \\
& 720 & 0.348 & 0.381 & 0.369 & 0.378 & \underline{0.337} & \textbf{0.333} & \textbf{0.312} & \underline{0.334} & 0.339 & 0.341 & 0.358 & 0.347 & 0.356 & 0.349 & 0.365 & 0.359 & 0.345 & 0.382 \\
& Avg & \underline{0.235} & 0.286 & 0.253 & 0.297 & 0.237 & \textbf{0.260} & \textbf{0.226} & \underline{0.262} & 0.240 & 0.272 & 0.258 & 0.278 & 0.265 & 0.286 & 0.259 & 0.287 & 0.265 & 0.315 \\
\midrule
        \multirow{5}{*}{\rotatebox{90}{Solar}}
        & 96 & \textbf{0.160} & 0.247 & 0.187 & 0.245 & 0.179 & \textbf{0.191} & \underline{0.171} & \underline{0.231} & 0.189 & 0.259 & 0.203 & 0.237 & 0.265 & 0.323 & 0.373 & 0.358 & 0.290 & 0.378 \\
& 192 & \underline{0.200} & \underline{0.256} & \textbf{0.199} & 0.257 & 0.209 & \textbf{0.213} & 0.218 & 0.263 & 0.222 & 0.283 & 0.233 & 0.261 & 0.288 & 0.332 & 0.397 & 0.376 & 0.320 & 0.398 \\
& 336 & \underline{0.221} & 0.278 & 0.239 & 0.277 & 0.231 & \textbf{0.229} & \textbf{0.212} & \underline{0.269} & 0.231 & 0.292 & 0.248 & 0.273 & 0.301 & 0.339 & 0.420 & 0.380 & 0.353 & 0.415 \\
& 720 & \textbf{0.212} & 0.278 & 0.233 & 0.273 & 0.241 & \textbf{0.236} & \underline{0.212} & \underline{0.270} & 0.223 & 0.285 & 0.249 & 0.275 & 0.295 & 0.336 & 0.420 & 0.381 & 0.357 & 0.413 \\
& Avg & \textbf{0.198} & 0.265 & 0.215 & 0.263 & 0.215 & \textbf{0.217} & \underline{0.203} & \underline{0.258} & 0.216 & 0.280 & 0.233 & 0.262 & 0.287 & 0.333 & 0.403 & 0.374 & 0.330 & 0.401 \\
\midrule
        \multirow{5}{*}{\rotatebox{90}{Electricity}}
        & 96 & 0.146 & 0.250 & 0.146 & 0.251 & \textbf{0.131} & \textbf{0.221} & \underline{0.135} & \underline{0.222} & 0.153 & 0.247 & 0.148 & 0.240 & 0.190 & 0.296 & 0.168 & 0.272 & 0.210 & 0.302 \\
& 192 & 0.161 & 0.265 & 0.170 & 0.271 & \underline{0.150} & \underline{0.238} & \textbf{0.147} & \textbf{0.235} & 0.166 & 0.256 & 0.162 & 0.253 & 0.199 & 0.304 & 0.184 & 0.322 & 0.210 & 0.305 \\
& 336 & 0.179 & 0.286 & 0.184 & 0.287 & \underline{0.165} & \underline{0.254} & \textbf{0.164} & \textbf{0.245} & 0.185 & 0.277 & 0.178 & 0.269 & 0.217 & 0.319 & 0.198 & 0.300 & 0.223 & 0.319 \\
& 720 & 0.221 & 0.320 & 0.233 & 0.336 & \textbf{0.191} & \textbf{0.279} & \underline{0.212} & \underline{0.310} & 0.225 & 0.310 & 0.225 & 0.317 & 0.258 & 0.352 & 0.220 & 0.320 & 0.258 & 0.350 \\
& Avg & 0.177 & 0.280 & 0.183 & 0.286 & \textbf{0.159} & \textbf{0.248} & \underline{0.165} & \underline{0.253} & 0.182 & 0.273 & 0.178 & 0.270 & 0.216 & 0.318 & 0.193 & 0.304 & 0.225 & 0.319 \\
\midrule
        \multirow{5}{*}{\rotatebox{90}{ETTh1}}
        & 96 & 0.370 & \underline{0.398} & 0.385 & 0.411 & \textbf{0.360} & \textbf{0.382} & \underline{0.361} & 0.403 & 0.375 & 0.400 & 0.386 & 0.405 & 0.460 & 0.447 & 0.384 & 0.402 & 0.397 & 0.412 \\
& 192 & 0.421 & 0.430 & 0.441 & 0.447 & \textbf{0.416} & \textbf{0.414} & \underline{0.416} & 0.441 & 0.429 & \underline{0.421} & 0.441 & 0.512 & 0.477 & 0.429 & 0.436 & 0.429 & 0.446 & 0.441 \\
& 336 & \underline{0.449} & 0.447 & 0.489 & 0.474 & 0.457 & \underline{0.438} & \textbf{0.430} & \textbf{0.434} & 0.484 & 0.458 & 0.487 & 0.458 & 0.546 & 0.496 & 0.491 & 0.469 & 0.489 & 0.467 \\
& 720 & \textbf{0.458} & 0.471 & 0.537 & 0.519 & \underline{0.463} & \underline{0.462} & 0.467 & \textbf{0.451} & 0.498 & 0.482 & 0.503 & 0.491 & 0.544 & 0.517 & 0.521 & 0.500 & 0.513 & 0.510 \\
& Avg & 0.425 & 0.437 & 0.463 & 0.463 & \underline{0.424} & \textbf{0.424} & \textbf{0.419} & \underline{0.432} & 0.447 & 0.440 & 0.454 & 0.467 & 0.507 & 0.472 & 0.458 & 0.450 & 0.461 & 0.458 \\
\midrule
        \multirow{5}{*}{\rotatebox{90}{ETTh2}}
        & 96 & 0.287 & 0.348 & 0.295 & 0.349 & \underline{0.284} & \underline{0.329} & \textbf{0.276} & \textbf{0.328} & 0.289 & 0.341 & 0.297 & 0.349 & 0.308 & 0.355 & 0.340 & 0.374 & 0.340 & 0.394 \\
& 192 & 0.382 & 0.407 & 0.407 & 0.420 & \underline{0.360} & \textbf{0.379} & \textbf{0.342} & \underline{0.379} & 0.372 & 0.392 & 0.380 & 0.400 & 0.393 & 0.405 & 0.402 & 0.414 & 0.482 & 0.479 \\
& 336 & 0.441 & 0.452 & 0.464 & 0.465 & 0.409 & 0.415 & \textbf{0.346} & \textbf{0.398} & \underline{0.386} & \underline{0.414} & 0.428 & 0.432 & 0.427 & 0.436 & 0.452 & 0.452 & 0.591 & 0.541 \\
& 720 & 0.439 & 0.471 & 0.491 & 0.503 & 0.415 & \underline{0.431} & \textbf{0.392} & \textbf{0.415} & \underline{0.412} & 0.434 & 0.427 & 0.445 & 0.436 & 0.450 & 0.462 & 0.468 & 0.839 & 0.661 \\
& Avg & 0.387 & 0.420 & 0.414 & 0.434 & 0.367 & \underline{0.389} & \textbf{0.339} & \textbf{0.380} & \underline{0.365} & 0.395 & 0.383 & 0.407 & 0.391 & 0.412 & 0.414 & 0.427 & 0.563 & 0.519 \\
\midrule
        \multirow{5}{*}{\rotatebox{90}{ETTm1}}
       & 96 & \textbf{0.297} & 0.346 & 0.314 & 0.364 & \underline{0.302} & \textbf{0.334} & 0.310 & \underline{0.334} & 0.320 & 0.357 & 0.334 & 0.368 & 0.352 & 0.374 & 0.338 & 0.375 & 0.346 & 0.374 \\
& 192 & \textbf{0.347} & 0.382 & 0.363 & 0.394 & 0.357 & \underline{0.363} & \underline{0.348} & \textbf{0.362} & 0.361 & 0.381 & 0.390 & 0.393 & 0.374 & 0.387 & 0.374 & 0.387 & 0.382 & 0.391 \\
& 336 & 0.388 & 0.409 & 0.408 & 0.427 & \underline{0.387} & \textbf{0.385} & \textbf{0.376} & \underline{0.391} & 0.390 & 0.404 & 0.426 & 0.420 & 0.421 & 0.414 & 0.410 & 0.411 & 0.415 & 0.415 \\
& 720 & \underline{0.441} & 0.446 & 0.462 & 0.456 & 0.452 & \underline{0.426} & \textbf{0.440} & \textbf{0.423} & 0.454 & 0.441 & 0.491 & 0.459 & 0.462 & 0.449 & 0.478 & 0.450 & 0.473 & 0.451 \\
& Avg & \textbf{0.368} & 0.396 & 0.387 & 0.410 & 0.375 & \textbf{0.377} & \underline{0.369} & \underline{0.378} & 0.381 & 0.396 & 0.410 & 0.410 & 0.402 & 0.406 & 0.400 & 0.406 & 0.404 & 0.408 \\
\midrule
        \multirow{5}{*}{\rotatebox{90}{ETTm2}}
        & 96 & \textbf{0.166} & 0.254 & 0.171 & 0.259 & \underline{0.169} & \underline{0.249} & 0.170 & \textbf{0.245} & 0.175 & 0.258 & 0.180 & 0.264 & 0.183 & 0.270 & 0.187 & 0.267 & 0.193 & 0.293 \\
& 192 & \underline{0.231} & 0.299 & 0.239 & 0.301 & 0.232 & \textbf{0.290} & \textbf{0.229} & \underline{0.291} & 0.237 & 0.299 & 0.250 & 0.309 & 0.255 & 0.314 & 0.249 & 0.309 & 0.284 & 0.361 \\
& 336 & \underline{0.298} & 0.345 & 0.301 & 0.348 & \textbf{0.291} & \textbf{0.328} & 0.303 & 0.343 & 0.298 & \underline{0.340} & 0.311 & 0.348 & 0.309 & 0.347 & 0.321 & 0.351 & 0.382 & 0.429 \\
& 720 & 0.412 & 0.417 & 0.417 & 0.417 & \underline{0.389} & \textbf{0.387} & \textbf{0.373} & 0.399 & 0.391 & \underline{0.396} & 0.412 & 0.407 & 0.412 & 0.404 & 0.408 & 0.403 & 0.558 & 0.525 \\
& Avg & 0.277 & 0.329 & 0.282 & 0.331 & \underline{0.270} & \textbf{0.314} & \textbf{0.269} & \underline{0.320} & 0.275 & 0.323 & 0.288 & 0.332 & 0.290 & 0.334 & 0.291 & 0.333 & 0.354 & 0.402 \\
\midrule
        \multirow{5}{*}{\rotatebox{90}{PEMS03}}
        & 12 & \underline{0.061} & \underline{0.165} & 0.074 & 0.182 & \textbf{0.060} & \textbf{0.159} & 0.097 & 0.208 & 0.076 & 0.188 & 0.071 & 0.174 & 0.099 & 0.216 & 0.085 & 0.192 & 0.122 & 0.243 \\
& 24 & \underline{0.079} & \underline{0.183} & 0.110 & 0.220 & \textbf{0.078} & \textbf{0.179} & 0.120 & 0.230 & 0.113 & 0.226 & 0.093 & 0.201 & 0.142 & 0.259 & 0.118 & 0.223 & 0.201 & 0.317 \\
& 48 & \textbf{0.096} & \textbf{0.208} & 0.158 & 0.260 & \underline{0.104} & \underline{0.210} & 0.170 & 0.272 & 0.191 & 0.292 & 0.125 & 0.236 & 0.211 & 0.319 & 0.155 & 0.260 & 0.228 & 0.317 \\
& 96 & \textbf{0.135} & \underline{0.258} & 0.167 & 0.264 & \underline{0.140} & \textbf{0.247} & 0.274 & 0.342 & 0.288 & 0.363 & 0.164 & 0.275 & 0.269 & 0.370 & 0.228 & 0.317 & 0.457 & 0.515 \\
& Avg & \textbf{0.093} & \underline{0.204} & 0.127 & 0.232 & \underline{0.096} & \textbf{0.199} & 0.165 & 0.263 & 0.167 & 0.267 & 0.113 & 0.222 & 0.180 & 0.291 & 0.147 & 0.248 & 0.278 & 0.375 \\
\midrule
        \multirow{5}{*}{\rotatebox{90}{PEMS04}}
        & 12 & \underline{0.069} & \underline{0.167} & 0.074 & 0.186 & \textbf{0.068} & \textbf{0.163} & 0.099 & 0.214 & 0.092 & 0.204 & 0.078 & 0.183 & 0.105 & 0.224 & 0.087 & 0.195 & 0.148 & 0.272 \\
& 24 & \textbf{0.078} & \underline{0.185} & 0.092 & 0.205 & \underline{0.079} & \textbf{0.176} & 0.115 & 0.231 & 0.128 & 0.243 & 0.095 & 0.205 & 0.153 & 0.275 & 0.103 & 0.215 & 0.224 & 0.340 \\
& 48 & \textbf{0.091} & \underline{0.199} & 0.097 & 0.209 & \underline{0.095} & \textbf{0.197} & 0.144 & 0.261 & 0.213 & 0.315 & 0.120 & 0.233 & 0.229 & 0.339 & 0.136 & 0.250 & 0.355 & 0.437 \\
& 96 & \textbf{0.108} & \textbf{0.221} & \underline{0.109} & \underline{0.224} & 0.122 & 0.226 & 0.185 & 0.297 & 0.307 & 0.384 & 0.150 & 0.262 & 0.291 & 0.389 & 0.190 & 0.303 & 0.452 & 0.504 \\
& Avg & \textbf{0.087} & \underline{0.193} & 0.093 & 0.206 & \underline{0.091} & \textbf{0.191} & 0.136 & 0.251 & 0.185 & 0.287 & 0.111 & 0.221 & 0.195 & 0.307 & 0.129 & 0.241 & 0.295 & 0.388 \\
\midrule
        \multirow{5}{*}{\rotatebox{90}{PEMS08}}
        & 12 & \underline{0.069} & \underline{0.170} & 0.074 & 0.175 & \textbf{0.068} & \textbf{0.159} & 0.119 & 0.222 & 0.091 & 0.201 & 0.079 & 0.182 & 0.168 & 0.232 & 0.112 & 0.212 & 0.154 & 0.276 \\
& 24 & \underline{0.091} & \underline{0.193} & 0.117 & 0.223 & \textbf{0.089} & \textbf{0.178} & 0.149 & 0.249 & 0.137 & 0.246 & 0.115 & 0.219 & 0.224 & 0.281 & 0.141 & 0.238 & 0.248 & 0.353 \\
& 48 & \textbf{0.119} & \underline{0.211} & 0.146 & 0.252 & \underline{0.123} & \textbf{0.204} & 0.206 & 0.292 & 0.265 & 0.343 & 0.186 & 0.235 & 0.321 & 0.354 & 0.198 & 0.283 & 0.440 & 0.470 \\
& 96 & \underline{0.181} & \underline{0.259} & 0.206 & 0.277 & \textbf{0.173} & \textbf{0.236} & 0.329 & 0.355 & 0.410 & 0.407 & 0.221 & 0.267 & 0.408 & 0.417 & 0.320 & 0.351 & 0.674 & 0.565 \\
& Avg & \underline{0.115} & \underline{0.208} & 0.136 & 0.232 & \textbf{0.113} & \textbf{0.194} & 0.201 & 0.280 & 0.226 & 0.299 & 0.150 & 0.226 & 0.280 & 0.321 & 0.193 & 0.271 & 0.379 & 0.416 \\
        \midrule
        \multicolumn{2}{c|}{Avg Rank} & 2.3 & 3.7 & 4.8 & 5.6 & 2.1 & 1.4 & 3.2 & 3.3 & 4.7 & 4.4 & 5.3 & 4.5 & 7.4 & 7.1 & 6.4 & 6.1 & 8.3 & 8.5 \\
\multicolumn{2}{c|}{Top-1} & 14 & 2 & 1 & 0 & 11 & 27 & 16 & 13 & 0 & 0 & 0 & 0 & 0 & 0 & 0 & 0 & 0 & 0 \\
        \bottomrule
      \end{tabular}
      }
    \end{small}
  \end{center}
\end{table*}

\begin{table*}[ht]
  \caption{Ablation study on CAPS components. All models use $L=96$. \textbf{Bold}: best, \underline{underline}: second best.}
  \label{tab:ablation}
  \begin{center}
    \begin{small}
      \setlength{\tabcolsep}{4pt}
      \begin{tabular}{c|c|cc|cc|cc|cc}
        \toprule
        \multicolumn{2}{c|}{\textbf{Model}} & \multicolumn{2}{c|}{\makecell{\textbf{CAPS} \\ \scriptsize (Full)}} & \multicolumn{2}{c|}{\makecell{\textbf{w/o Path 1} \\ \scriptsize (Riemann)}} & \multicolumn{2}{c|}{\makecell{\textbf{w/o Path 2} \\ \scriptsize (Prefix)}} & \multicolumn{2}{c}{\makecell{\textbf{w/o Path 3} \\ \scriptsize (Clock)}} \\
        \midrule
        \multicolumn{2}{c|}{\textbf{Metric}} & MSE & MAE & MSE & MAE & MSE & MAE & MSE & MAE \\
        \midrule
        \multirow{5}{*}{\rotatebox{90}{Weather}} & 96 & \textbf{0.144} & \textbf{0.202} & 0.148 & \underline{0.204} & 0.145 & \underline{0.204} & \textbf{0.144} & \underline{0.204} \\
        & 192 & \textbf{0.198} & \underline{0.261} & \textbf{0.198} & \textbf{0.258} & \textbf{0.198} & 0.263 & 0.204 & 0.263 \\
        & 336 & \textbf{0.251} & \textbf{0.300} & 0.273 & 0.325 & 0.276 & \underline{0.316} & \underline{0.268} & 0.321 \\
        & 720 & \textbf{0.348} & \underline{0.381} & 0.365 & \underline{0.381} & 0.368 & 0.385 & \underline{0.360} & \textbf{0.371} \\
        & Avg & \textbf{0.235} & \textbf{0.286} & 0.246 & 0.292 & 0.247 & 0.292 & \underline{0.244} & \underline{0.290} \\
        \midrule
        \multirow{5}{*}{\rotatebox{90}{ETTh1}} & 96 & \textbf{0.370} & 0.398 & 0.371 & \textbf{0.393} & 0.371 & 0.394 & \textbf{0.370} & \textbf{0.393} \\
        & 192 & 0.421 & 0.430 & \textbf{0.419} & \textbf{0.424} & \textbf{0.419} & \underline{0.425} & 0.420 & 0.427 \\
        & 336 & \textbf{0.449} & 0.447 & 0.453 & 0.449 & \underline{0.452} & \textbf{0.443} & 0.453 & \underline{0.444} \\
        & 720 & \textbf{0.458} & 0.471 & 0.462 & \textbf{0.464} & \underline{0.461} & \underline{0.465} & 0.478 & 0.472 \\
        & Avg & \textbf{0.425} & 0.437 & \underline{0.426} & \underline{0.433} & \underline{0.426} & \textbf{0.432} & 0.430 & 0.434 \\
        \midrule
        \multirow{5}{*}{\rotatebox{90}{ETTh2}} & 96 & \textbf{0.287} & 0.348 & \underline{0.292} & \textbf{0.345} & 0.294 & \underline{0.346} & 0.295 & 0.349 \\
        & 192 & \textbf{0.382} & 0.407 & 0.385 & \textbf{0.401} & \underline{0.384} & 0.403 & 0.388 & \textbf{0.401} \\
        & 336 & \underline{0.441} & 0.452 & 0.443 & \textbf{0.444} & 0.444 & 0.457 & \textbf{0.436} & \underline{0.446} \\
        & 720 & 0.439 & 0.471 & \textbf{0.434} & \textbf{0.455} & 0.439 & 0.459 & \underline{0.436} & \textbf{0.455} \\
        & Avg & \textbf{0.387} & 0.420 & \underline{0.389} & \textbf{0.411} & 0.390 & 0.416 & \underline{0.389} & \underline{0.413} \\
        \midrule
        \multirow{5}{*}{\rotatebox{90}{ETTm1}} & 96 & \textbf{0.297} & \textbf{0.346} & \underline{0.305} & \underline{0.353} & 0.309 & 0.359 & 0.310 & 0.358 \\
        & 192 & 0.347 & 0.382 & \textbf{0.345} & \textbf{0.381} & 0.347 & 0.382 & \underline{0.346} & \textbf{0.381} \\
        & 336 & \underline{0.388} & \textbf{0.409} & \underline{0.388} & \underline{0.410} & \underline{0.388} & \underline{0.410} & \textbf{0.385} & \underline{0.410} \\
        & 720 & \textbf{0.441} & \textbf{0.446} & \underline{0.447} & 0.448 & 0.448 & \textbf{0.446} & 0.450 & \textbf{0.446} \\
        & Avg & \textbf{0.368} & \textbf{0.396} & \underline{0.371} & \underline{0.398} & 0.373 & 0.399 & 0.373 & 0.399 \\
        \midrule
        \multirow{5}{*}{\rotatebox{90}{ETTm2}} & 96 & \textbf{0.166} & \textbf{0.254} & \underline{0.168} & 0.258 & 0.171 & 0.260 & 0.172 & \underline{0.257} \\
        & 192 & \textbf{0.231} & \textbf{0.299} & 0.237 & 0.306 & 0.242 & 0.309 & \underline{0.236} & \underline{0.304} \\
        & 336 & \textbf{0.298} & \textbf{0.345} & \underline{0.300} & \underline{0.346} & \underline{0.300} & \underline{0.346} & 0.302 & 0.354 \\
        & 720 & \textbf{0.412} & \textbf{0.417} & \underline{0.414} & \textbf{0.417} & 0.416 & 0.418 & \underline{0.414} & 0.418 \\
        & Avg & \textbf{0.277} & \textbf{0.329} & \underline{0.280} & \underline{0.332} & 0.282 & 0.333 & 0.281 & 0.333 \\
        \midrule
        \multicolumn{2}{c|}{Avg $\Delta$} & \multicolumn{2}{c|}{--} & -1.17\% & 0.05\% & -1.54\% & -0.32\% & -1.46\% & -0.11\% \\
        \bottomrule
      \end{tabular}
    \end{small}
  \end{center}
\end{table*}

\begin{table*}[ht]
  \caption{Attention mechanism comparison. All models use $L=96$. \textbf{Bold}: best, \underline{underline}: second best.}
  \label{tab:attention-ablation-full}
  \begin{center}
    \begin{small}
      \setlength{\tabcolsep}{3pt}
      \begin{tabular}{c|c|cc|cc|cc|cc|cc}
        \toprule
        \multicolumn{2}{c|}{\textbf{Model}} & \multicolumn{2}{c|}{\makecell{\textbf{CAPS} \\ \scriptsize (Linear)}} & \multicolumn{2}{c|}{\makecell{\textbf{CAPS} \\ \scriptsize (Softmax)}} & \multicolumn{2}{c|}{\makecell{\textbf{Attn} \\ \scriptsize +RoPE}} & \multicolumn{2}{c|}{\makecell{\textbf{Pure} \\ \scriptsize Softmax}} & \multicolumn{2}{c}{\makecell{\textbf{LinAttn} \\ \scriptsize +RoPE}} \\
        \midrule
        \multicolumn{2}{c|}{\textbf{Metric}} & MSE & MAE & MSE & MAE & MSE & MAE & MSE & MAE & MSE & MAE \\
        \midrule
        \multirow{5}{*}{\rotatebox{90}{Weather}} & 96 & \textbf{0.144} & \textbf{0.202} & 0.157 & 0.221 & \underline{0.146} & \underline{0.206} & \underline{0.146} & 0.209 & 0.152 & 0.211 \\
        & 192 & \textbf{0.198} & \underline{0.261} & 0.211 & 0.273 & 0.201 & 0.262 & \textbf{0.198} & \textbf{0.260} & 0.210 & 0.270 \\
        & 336 & \textbf{0.251} & \textbf{0.300} & 0.279 & 0.328 & 0.270 & \underline{0.317} & \underline{0.269} & \underline{0.317} & 0.280 & 0.329 \\
        & 720 & \textbf{0.348} & 0.381 & 0.364 & 0.381 & 0.360 & \textbf{0.374} & \underline{0.359} & \underline{0.378} & 0.369 & \underline{0.378} \\
        & Avg & \textbf{0.235} & \textbf{0.286} & 0.253 & 0.301 & 0.244 & \underline{0.290} & \underline{0.243} & 0.291 & 0.253 & 0.297 \\
        \midrule
        \multirow{5}{*}{\rotatebox{90}{Solar}} & 96 & \textbf{0.160} & 0.247 & \underline{0.180} & \textbf{0.239} & 0.190 & 0.268 & 0.183 & 0.249 & 0.187 & \underline{0.245} \\
        & 192 & \underline{0.200} & \textbf{0.256} & 0.210 & 0.284 & 0.214 & 0.279 & 0.209 & \textbf{0.256} & \textbf{0.199} & 0.257 \\
        & 336 & \textbf{0.221} & 0.278 & \underline{0.227} & \textbf{0.273} & 0.233 & 0.294 & 0.232 & 0.295 & 0.239 & \underline{0.277} \\
        & 720 & \textbf{0.212} & 0.278 & 0.232 & \underline{0.277} & \underline{0.218} & 0.286 & 0.219 & 0.278 & 0.233 & \textbf{0.273} \\
        & Avg & \textbf{0.198} & \underline{0.265} & 0.212 & 0.268 & 0.214 & 0.282 & \underline{0.211} & 0.270 & 0.215 & \textbf{0.263} \\
        \midrule
        \multirow{5}{*}{\rotatebox{90}{Electricity}} & 96 & \textbf{0.146} & \underline{0.250} & 0.148 & 0.252 & 0.148 & 0.253 & \textbf{0.146} & \textbf{0.249} & \textbf{0.146} & 0.251 \\
        & 192 & \textbf{0.161} & \underline{0.265} & 0.175 & 0.273 & \underline{0.163} & 0.266 & 0.164 & \textbf{0.263} & 0.170 & 0.271 \\
        & 336 & \textbf{0.179} & \textbf{0.286} & 0.183 & 0.290 & \underline{0.182} & 0.288 & 0.187 & 0.289 & 0.184 & \underline{0.287} \\
        & 720 & \textbf{0.221} & \textbf{0.320} & 0.228 & 0.329 & \underline{0.223} & \underline{0.325} & 0.232 & 0.331 & 0.233 & 0.336 \\
        & Avg & \textbf{0.177} & \textbf{0.280} & 0.184 & 0.286 & \underline{0.179} & \underline{0.283} & 0.182 & \underline{0.283} & 0.183 & 0.286 \\
        \midrule
        \multirow{5}{*}{\rotatebox{90}{ETTh1}} & 96 & \textbf{0.370} & \textbf{0.398} & 0.379 & 0.408 & 0.381 & 0.407 & \underline{0.375} & \underline{0.406} & 0.385 & 0.411 \\
        & 192 & \textbf{0.421} & \textbf{0.430} & 0.431 & 0.440 & 0.430 & 0.439 & \underline{0.423} & \underline{0.433} & 0.441 & 0.447 \\
        & 336 & \textbf{0.449} & \textbf{0.447} & 0.473 & \underline{0.461} & 0.472 & 0.464 & \underline{0.466} & \underline{0.461} & 0.489 & 0.474 \\
        & 720 & \textbf{0.458} & \textbf{0.471} & \underline{0.481} & \underline{0.488} & 0.511 & 0.512 & 0.489 & 0.495 & 0.537 & 0.519 \\
        & Avg & \textbf{0.425} & \textbf{0.437} & 0.441 & \underline{0.449} & 0.449 & 0.456 & \underline{0.438} & \underline{0.449} & 0.463 & 0.463 \\
        \midrule
        \multirow{5}{*}{\rotatebox{90}{ETTh2}} & 96 & \textbf{0.287} & \underline{0.348} & 0.300 & 0.354 & 0.294 & \textbf{0.347} & \underline{0.291} & 0.350 & 0.295 & 0.349 \\
        & 192 & \textbf{0.382} & \textbf{0.407} & 0.399 & 0.418 & \underline{0.393} & \underline{0.417} & 0.412 & 0.430 & 0.407 & 0.420 \\
        & 336 & \textbf{0.441} & \textbf{0.452} & \underline{0.451} & \underline{0.458} & 0.470 & 0.480 & 0.461 & 0.463 & 0.464 & 0.465 \\
        & 720 & \textbf{0.439} & \textbf{0.471} & 0.450 & 0.477 & \underline{0.445} & \underline{0.472} & 0.497 & 0.501 & 0.491 & 0.503 \\
        & Avg & \textbf{0.387} & \textbf{0.420} & \underline{0.400} & \underline{0.427} & 0.401 & 0.429 & 0.415 & 0.436 & 0.414 & 0.434 \\
        \midrule
        \multirow{5}{*}{\rotatebox{90}{ETTm1}} & 96 & \textbf{0.297} & \textbf{0.346} & 0.323 & 0.368 & \underline{0.306} & \underline{0.355} & 0.313 & 0.359 & 0.314 & 0.364 \\
        & 192 & \textbf{0.347} & \textbf{0.382} & 0.367 & 0.394 & 0.357 & 0.387 & \underline{0.351} & \underline{0.383} & 0.363 & 0.394 \\
        & 336 & \textbf{0.388} & \textbf{0.409} & 0.403 & 0.419 & 0.399 & 0.417 & \underline{0.393} & \underline{0.411} & 0.408 & 0.427 \\
        & 720 & \textbf{0.441} & \textbf{0.446} & 0.462 & 0.453 & 0.454 & 0.459 & \underline{0.451} & \underline{0.449} & 0.462 & 0.456 \\
        & Avg & \textbf{0.368} & \textbf{0.396} & 0.389 & 0.409 & 0.379 & 0.405 & \underline{0.377} & \underline{0.401} & 0.387 & 0.410 \\
        \midrule
        \multirow{5}{*}{\rotatebox{90}{ETTm2}} & 96 & \textbf{0.166} & \textbf{0.254} & 0.170 & 0.257 & 0.172 & 0.257 & \underline{0.169} & \underline{0.255} & 0.171 & 0.259 \\
        & 192 & \textbf{0.231} & \textbf{0.299} & 0.240 & 0.304 & 0.244 & 0.306 & 0.242 & 0.305 & \underline{0.239} & \underline{0.301} \\
        & 336 & \textbf{0.298} & \textbf{0.345} & 0.308 & 0.351 & 0.317 & 0.352 & 0.306 & 0.352 & \underline{0.301} & \underline{0.348} \\
        & 720 & \textbf{0.412} & \underline{0.417} & 0.420 & 0.421 & \underline{0.416} & \underline{0.417} & \underline{0.416} & \textbf{0.416} & 0.417 & \underline{0.417} \\
        & Avg & \textbf{0.277} & \textbf{0.329} & 0.285 & 0.333 & 0.287 & 0.333 & 0.283 & 0.332 & \underline{0.282} & \underline{0.331} \\
        \midrule
        \multicolumn{2}{c|}{Average} & \textbf{0.313} & \textbf{0.355} & 0.325 & 0.361 & 0.324 & 0.361 & \underline{0.321} & \underline{0.359} & 0.323 & 0.360 \\
        \bottomrule
      \end{tabular}
    \end{small}
  \end{center}
\end{table*}

\FloatBarrier
\subsection{Statement of LLM Usage}
\label{sec:llm-usage}

LLMs were used to assist with writing-related tasks including grammar checking, wording adjustments, text formatting, and equation formatting. LLMs were also used to search for existing methods and references. All cited literature was read by the authors directly. During experiments, LLMs assisted with generating and debugging code. LLMs were not used for defining research problems, proposing ideas, designing methodologies, or developing the model architecture.


\end{document}